\documentclass[runningheads]{llncs}
\usepackage{amssymb}
\usepackage{amsmath}
\usepackage{subeqnarray}
\usepackage{epstopdf}
\usepackage{shuffle}
\usepackage{tikz}
\usepackage{graphicx}
\usepackage[hyphens]{url}
\usepackage{microtype}
\usepackage{url}
\def\P{{\mathchoice {\hbox{$\sf\textstyle P\kern-0.4em Z$}}
{\hbox{$\sf\textstyle P\kern-0.4em P$}}
{\hbox{$\sf\scriptstyle P\kern-0.3em P$}}
{\hbox{$\sf\scriptscriptstyle P\kern-0.2em P$}}}}

\usepackage{epsf}
\usepackage{enumerate}
\usepackage{enumitem}
\usepackage[numbers]{natbib}

\def\P{{\mathchoice {\hbox{$\sf\textstyle P\kern-0.4em Z$}}
{\hbx{$\sf\textstyle P\kern-0.4em P$}}
{\hbox{$\sf\scriptstyle P\kern-0.3em P$}}
{\hbox{$\sf\scriptscriptstyle P\kern-0.2em P$}}}}

\let\Rightarrow=\Rightarrow

\newtheorem{fait}{Claim}

\def\ex#1\par{\par\noindent\begin{exemple} \nopagebreak \strut \rm #1 \end{exemple}}
\def\thm#1\par{\medskip\par\noindent\begin{theorem} \strut \sl #1 \end{theorem}\par}
\def\propo#1\par{\medskip\par\noindent\begin{proposition} \strut \sl #1 \end{proposition} \par}
\def\proo#1\par{\medskip\par\noindent{\it Proof.} \strut \rm #1 $\Box$ \par}
\def\prob#1\par{\medskip\par\noindent\begin{prob} \strut \sl #1 \end{prob} \par}
\def\cor#1\par{\medskip\par\noindent\begin{corollary} \strut \sl #1 \end{corollary}\par}
\def\lm#1\par{\medskip\par\noindent\begin{lemma} \strut \sl #1 \end{lemma}\par}
\def\defil#1\par{\medskip\par\noindent\begin{condit} \strut \sl #1\end{condit}\par}
\def\fct#1\par{\par\noindent\begin{fait}  \nopagebreak \strut #1 \end{fait}}
\def\fct#1\par{\par\noindent\begin{claim} \nopagebreak \strut #1 \end{claim}}
\def\defi#1\par{\medskip\par\noindent{\begin{defin} \strut  \sl #1 \end{defin}}\par}
\def\nota#1\par{\par\noindent\begin{notat} \nopagebreak  \strut #1 \end{notat}}

\def\rem#1\par{\par\noindent\begin{rema} \nopagebreak \strut \rm #1 \end{rema}}

\date{}

\begin{document}

\thispagestyle{empty}
%
\title{Gray Cycles of Maximum Length Related to $k$-Character Substitutions}
%
\titlerunning{
Gray cycles of maximum length related to $k$-character substitutions}
\author{Jean N\'eraud}
\institute{Universit\'e de Rouen, Laboratoire d'Informatique, de Traitement de l'Information et des Syst\`emes, 
 Avenue de l'Universit\'e, 76800 Saint-\'Etienne-du-Rouvray, France\\
\email{jean.neraud@univ-rouen.fr, neraud.jean@gmail.com}
\url{http:neraud.jean.free.fr; orcid: 0000-0002-9630-461X} 
}
\toctitle{v}
\tocauthor{J.~N\'eraud}
\authorrunning{J. N\'eraud}
\maketitle
\setcounter{footnote}{0}

%
\begin{abstract}
Given a word binary relation $\tau$ onto $A^*$ we define a $\tau$-Gray cycle over a finite language $X$ to be a permutation $\left(w_{[i]}\right)_{0\le i\le |X|-1}$ of $X\subseteq A^*$ such that each word $w_i$ is an image of the previous word $w_{i-1}$ by $\tau$.
In that framework, we introduce  the complexity measure $\lambda(n)$, equal to the largest  cardinality of a language $X$ having words of length at most $n$,
and such that some $\tau$-Gray cycle  over $X$  exists. The present paper is concerned with the  relation $\tau=\sigma_k$,
the so-called $k$-character substitution,
where $(u,v)$ belongs to $\sigma_k$ if, and only if, the Hamming distance of $u$ and $v$ is $k$.
We compute the bound $\lambda(n)$ for all cases of the alphabet cardinality and the argument $n$.
\end{abstract}
\section{Introduction}
\label{Intro}
In the framework of combinatorial algorithms,  one of the most documented questions consists in the development of methods 
in order to generate, exactly once, all the objects in some specific class \cite{LEH64}. 
Many topics are concerned by such a problem: suffice it to mention sequence counting \cite{BLPP99},
signal encoding \cite{L81}, and  data compression \cite{R86}.

The so-called {\it binary Gray codes} 
first appeared in \cite{G58}:
given a binary alphabet $A$ and some positive integer $n$,
they referred to  sequences  with maximum length of pairwise different $n$-tuples  of characters (that is, words in  $A^n$), provided that any pair of consecutive items  differ by exactly one character.
Shortly after, a similar study was drawn in the framework of non-binary alphabets \cite{C63,E84}.
With regard to other famous combinatorial classes of objects, the  term of {\it combinatorial Gray code}, for its part,
appeared in \cite{JWW80}:
actually,  the difference between successive items, although being fixed, need not to be small \cite{S97}. 
Generating all permutations of  a given $n$-element set
constitutes a noticeable example \cite{E73}.
Subsets of fixed size are also concerned \cite{EMcK84}, 
as well as cross-bifix-free words \cite{BBPSV14}, Debruijn sequences \cite{FM86}, set partitions \cite{K76}, necklaces \cite{RSW92}:
the list is  far to be exhaustive.
Combinatorial Gray sequences are often needed to be {\it cyclic} \cite{CDG92},
 in the sense that 
the initial term itself can be retrieved as  successor of the last one. 
Such a condition justifies the terminology of {\it Gray cycle} \cite[Sect. 7.2.1.1]{K05}.

In view of some formal framework, we notice that each of the sequences we mentioned above involves some word binary relation $\tau\subseteq A^*\times A^*$, where $A$ stands for a finite alphabet, and $A^*$ for the free monoid it generates.
For its part, the combinatorial class of objects can be modelled by some finite langage $X\subseteq A^*$. 
Given  a sequence of words we denote in square brackets the corresponding indices: this will allow us to make a difference from $w_i$, the character in {\it position} $i$ in a given word $w$.
We define a {\it cyclic Gray sequence over $X$, with respect to $\tau$} (for short: {\it $\tau$-Gray cycle over $X$}) as 
every finite sequence of words $\left(w_{[i]}\right)_{i\in [0,|X|-1]}$ satisfying each of the three following conditions:
\begin{enumerate}[label={\rm (G\arabic*)},leftmargin=2cm]
\item 
\label{iva}
For every word $x\in X$, some $i\in [0,|X|-1]$ exists such that we have $x=w_{[i]}$;
\item
\label{ivb}
 For every $i\in [1,|X|-1] $,  we have $w_{[i]}\in\tau\left(w_{[i-1]}\right)$; in addition, the condition $w_{[0]}\in\tau\left(w_{[|X|-1]}\right)$ holds;
\item 
\label{ivc} For every pair $i,j\in [0,|X|-1]$, $i\neq j$ implies $w_{[i]}\neq w_{[j]}$.
\end{enumerate}
With this definition, the set $X$ need not to be {\it uniform} that is, in the Gray cycle the terms may have a variable length.
For instance, given the alphabet $A=\{0,1\}$, take for $\tau$  the word binary relation $\Lambda_1$ which,
with every word $w$ associates all the strings located  within a {\it Levenshtein distance} of $1$ from $w$ (see e.g.  \cite{N21});
with such a relation, the  sequence $(0,00,01,11,10,1)$ is a $\Lambda_1$-Gray cycle over $X=A\cup A^2$.  

Actually, in addition to the topics we mentioned above, two other  famous topics are involved by those Gray cycles.
Firstly,  with regard to graph theory, a $\tau$-Gray cycle over $X$ exists if, and only if,
there is some hamiltonian circuit  in the  graph of the relation $\tau$ (see e.g. \cite{S97}).
Secondly, given a binary word relation $\tau\subseteq A^*\times A^*$, and given $X\subseteq A^*$,
if some $\tau$-Gray cycle exists over $X$, then $X$ is $\tau$-{\it closed}  \cite{N21} that is,  the inclusion $\tau(X)\subseteq X$ holds, where $\tau(X)$ stands for the set of the images of the words in $X$ under the relation $\tau$. 
Such closed sets actually constitute a special subfamily in the famous {\it dependence systems} (see \cite{C1965,JK1997}).
Notice that,  given a $\tau$-{\it closed}
set $X\subseteq A^*$, there do not necessarily exist non-empty $\tau$-Gray cycles over $X$. 
A typical example is provided by $\tau$ being $id_{A^*}$, the identity over $A^*$, with respect to which 
every finite set $X\subseteq A^*$ is closed;
however non-empty $\tau$-Gray cycle can exist only over singletons.

In the present paper, given a positive integer $n$, and denoting by $A^{\le n}$ the set of the words with length not greater than $n$, we consider 
the family of  all sequences 
that can be a $\tau$-Gray cycle over 
some subset $X$ of  $A^{\le n}$.
This is a natural question to study those sequences of maximum  length which, of course, correspond to subsets $X$ of maximum cardinality;
such a length, which we denote by $\lambda_{A,\tau}(n)$, means introducing some  complexity measure for the word binary relation $\tau$.
With regard to the preceding examples we have $\lambda_{\{0,1\},\Lambda_1}(2)=6$; moreover,  for every alphabet $A$ and positive integer $n$, the identity $\lambda_{A,id_{A^*}}(n)=1$ holds. 
We focus on  the case where $\tau$ is  $\sigma_k$, the so-called {\it $k$-character substitution}:
with every word with length at least $k$, say $w$, this relation associates all  the words $w'$, with $|w'|=|w|$, and such that the character $w'_i$ differs from $w_i$ in exactly $k$ values of  $i\in[1,|w|]$. 
As commented in \cite{JK1997,N21}, $\sigma_k$ has noticeable inference in the famous framework of {\it error detection}.
On the other hand,  by definition, $w'\in\sigma_k(w)$ implies $|w'|=|w|$ therefore,  if  there is some $\sigma_k$-Gray cycle over $X$, then $X$ is a uniform set. 
From this point of view, the classical Gray codes, which allow to generate all $n$-tuples over $A$,  
correspond to $\sigma_1$-Gray cycles over $A^n$, furthermore  
we have $\lambda_{A,\sigma_1}(n)=|A|^n$.
In addition, in the case where $A$ is  a binary alphabet,
it can be easily proved that, for every $n\ge3$,  we have $\lambda_{A,\sigma_2}(n)=2^{n-1}$ \cite[Exercice 8, p. 77]{K05}.
However, in the most general case, although an exhaustive description of $\sigma_k$-closed variable-length codes has been  provided in \cite{N21},
the question of computing some $\sigma_k$-Gray cycle of maximum length  has remained open.
In our paper we establish the following result:\\
\ \\
 {\bf Theorem}
 {\it  
Let $A$ be  a finite alphabet, $k\ge 1$, and $n\ge k$. 
Then exactly one of the following conditions  holds:

{\rm (i)} $|A|\ge 3$, $n\ge k$, and $\lambda_{A,\sigma_k}(n)=|A|^n$;

{\rm (ii)}  $|A|=2$,   $n=k$, and $\lambda_{A,\sigma_k}(n)=2$;

{\rm (iii)}  $|A|=2$,  $n\ge k+1$, $k$ is odd and  $\lambda_{A,\sigma_k}(n)=|A|^n$;

{\rm (iv)} $|A|=2$, $n\ge k+1$, $k$ is even, and $\lambda_{A,\sigma_k}(n)=|A|^{n-1}$.\\
In addition, in each case some $\sigma_k$-Gray cycle of maximum length can be explicitly computed.}

\smallbreak
We now shortly describe the contents of the paper. In  Sect. \ref{Prelim}, we recall the  two famous examples of
the  {\it binary} (resp., $|A|$-{\it ary}) {\it reflected Gray code}.
By applying some induction based methods, in  Sect. \ref{A>3} and  Sect.  \ref{S4}, these sequences allow to compute special families of $\sigma_k$-Gray cycles with maximum length.
In  Sect. \ref{Cons}, in the case where $A$ is a binary alphabet, and $k$ an even positive integer,
we  also compute a  family of  Gray cycles with maximum length; 
in addition some further development is raised.
\section{Preliminaries}
\label{Prelim}
Several definitions and notation  have already been fixed.
In  the whole paper, $A$ stands for some finite alphabet, with $|A|\ge 2$.
Given a word $w\in A^*$,  we denote by $|w|$ its length;  in addition, for every $a\in A$, we denote by $|w|_a$  the number of occurrences of the character (or character) $a$ in $w$.\\ 
\ \\
{\it The reflected binary Gray cycle}\\
Let $A=\{0,1\}$, and $n\ge 1$. The most famous example of $\sigma_1$-Gray cycle over $A^n$ is certainly the  so-called {\it reflected binary Gray code} (see e.g. \cite{G58} or \cite[p. 6]{K05}):
 in the present paper we denoted it by $g^{n,1}$. It can be defined by the recurrent sequence initialized  with  $g^{n,1}_{[0]}=0^n$, and satisfying the following property:
for every $i\in [1,|A|^n-1]$,  a unique integer $j\in [1,n]$ exists such that,  in both  words $g^{n,1}_{[i]}$ and  $g^{n,1}_{[i-1]}$ the corresponding characters in position $j$ differ;
in addition, the position $j$  is chosen to be maximum in such a way that $g^{n,1}_{[i]}\notin\{g^{n,1}_{[0]},\cdots,g^{n,1}_{[i-1]}\}$.
\begin{example}
\label{Ex1}
In what follows we provide a column representation of $g^{2,1}$ and $g^{3,1}$:
$$
{\small
\begin{array}{c}
 g^{2,1}\\ \overbrace{}\\00\\ 01\\11\\~10
\ \\ \ \\ \ \\ \ \\ \ \\
\end{array}
~~~~~~~~~~~~~~~~
\begin{array}{c}
 g^{3,1}
\\ \overbrace{~~~~~~}\\000\\ 001\\ 011\\ 010\\ 110\\ 111 \\101\\ 100\\ 
\end{array}
}
$$
\end{example}
By construction, for every $n\ge 1$, each of the following identities holds:
\begin{eqnarray}
\label{g-termes-remarquables}
g^{n,1}_{[0]}=0^n,~~g^{n,1}_{[1]}=0^{n-1}1,~~g^{n,1}_{[2^{n}-2]}=10^{n-2}1,~~g^{n,1}_{[2^{n}-1]}=10^{n-1}
\end{eqnarray}
{\it The  $|A|$-ary reflected Gray cycle}\\
The preceding construction can be extended in order to obtain the so-called {\it $|A|$-ary reflected Gray code}  \cite{C63,E84}, 
a $\sigma_1$-Gray cycle over $A^n$, which we denote by $h^{n,1}$.
Set $A=\{0, \cdots,p-1\}$ and denote  by $\theta$  the cyclic permutation $(0,1,\dots p-1)$. 
The sequence  $h^{n,1}$ is initialized with  $h^{n,1}_{[0]}=0^n$.
In addition, for every $i\in [1,|A|^n-1]$, a unique integer $j$ exists such that $c$ and  $d$, the characters respectively in position $j$ in $h^{n,1}_{[i-1]}$ and  $h^{n,1}_{[i]}$,
 satisfy both the following conditions:
(i) $d=\theta (c)$;
(ii) $j$ is the greatest integer in $[1,n]$ such that $h^{n,1}_{[i]}\notin\{h^{n,1}_{[0]},\cdots,h^{n,1}_{[i-1]}\}$.
\begin{example}
\label{Ex2}
For $A=\{0,1,2\}$ the sequence $h^{3,1}$ 
is the concatenation in this order of the three following subsequences:
$$
{\small
\begin{array}{c}
h^{3,1}
\\ \overbrace{~~~~~~}\\000\\ 001\\ 002\\ 012\\ 010\\ 011 \\021\\ 022\\  020
\end{array}
~~~~~~~~
\begin{array}{c}
\ \\ \ \\ 120\\ 121\\ 122\\ 102\\ 100\\ 101 \\111\\  112\\110
\end{array}
~~~~~~~~
\begin{array}{c}
\ \\ \ \\  210\\ 211\\ 212\\ 222\\ 220 \\221\\ 201\\  202\\200
\end{array}
}
$$
\end{example}
%
\section{The case where we have $k\ge 1$ and $|A|\ge 3$}
\label{A>3}
Let  $n\ge k\ge 1$, $p\ge 3$, and $A=\{0,1,\cdots,p-1\}$. We will indicate the construction of  a peculiar  $\sigma_k$-Gray cycle over $A^n$, namely  $ h^{n,k}$.  
This will be done by applying some induction over $k\ge 1$: in view of that we set $n_0=n-k+1$.
The starting point corresponds to $h^{n_0,1}$, the  $p$-ary reflected Gray code over  $A^{n_0}$ as reminded  in  Sect. \ref{Prelim}.
 For the induction stage, starting with some $\sigma_{k-1}$-Gray cycle over $A^{n-1}$, namely   $ h^{n-1,k-1}$, 
we compute the sequence $ h^{n,k}$ as indicated in what follows:
let $i\in [0,p^{n}-1]$, and let $q\in [0,p-1]$, $r\in [0,p^{n-1}-1]$ be  the unique pair of non-negative integers such that $i=qp^{n-1}+r$. We  set:
\begin{eqnarray}
\label{Gamma-A-ge3}
h^{n,k}_{[i]}=h^{n,k}_{[qp^{n-1}+r]}=
\theta^{q+r}(0) h^{n-1,k-1}_{[r]}
\end{eqnarray}
As illustrated by Example \ref{E3}, the resulting sequence  $ h^{n,k}$ is actually the concatenation in this order of $p$ subsequences namely $C_0, \dots,  C_{p-1}$,
with $C_q=\left( h^{n,k}_{[qp^{n-1}+r]}\right)_{0\le r\le p^{n-1}-1}$, for each $q\in [0,p-1]$.
Since $\theta$ is one-to-one, given a pair of different integers $q,q'\in [0,p-1]$,
for every  $r\in [0,p^{n-1}-1]$,  in each of  the  subsequences $C_q$, $C_{q'}$, the words  $h^{n,k}_{[qp^{n-1}+r]}$ and $h^{n,k}_{[q'p^{n-1}+r]}$ only differ in their initial characters,
which  respectively are $\theta^{q+r}(0)$ and $\theta^{q'+r}(0)$.
In addition, since $h^{n-1,k-1}$ is a $\sigma_{k-1}$-Gray cycle over $A^{n-1}$, we have $\left | h^{n,k}\right |=p\left | h^{n-1,k-1}\right |=p^n$. 
\begin{example}
\label{E3} 
{\rm Let $A=\{0,1,2\}$,  $n=3$, $k=2$, thus  $p=3$, $n_0=2$. By starting with the sequence $h^{n-1,k-1}=h^{2,1}$,
$ h^{n,k}$ is  the concatenation of $C_0$, $C_1$, and $C_2$ : 
$$
{\small
\begin{array} {c}
 h^{n-1,k-1}\\ \overbrace{~~~~~}\\00\\ 01\\ 02\\ 12\\ 10\\ 11 \\21\\ 22\\~20
\ \\ \ \\
\end{array}
~~~~~~~~~~~~~~~~
\begin{array}{c}
 h^{n,k}
\\ \overbrace{~~~~~~}\\000\\ 101\\ 202\\ 012\\ 110\\ 211 \\021\\ 122\\ 220\\
\\
\end{array}
~~~~~
\begin{array}{c}
\ \\100 \\ 201\\ 002\\ 112\\ 210\\ 011 \\121\\ 222\\ 020
\end{array}
~~~~~
\begin{array}{c}
\ \\ 200\\ 001\\ 102\\ 212\\ 010\\ 111\\ 221\\ 022\\ 120\\
\end{array}
}
$$
}
\end{example}
\begin{proposition}
\label{Cycle-Age3-n-k}
$h^{n,k}$
is a $\sigma_k$-Gray cycle over $A^n$.
\end{proposition}
\begin{proof}
We argue by induction over $k\ge 1$.
With regard to the base case, as indicated above  $h^{n_0,1}$
 is the $|A|$-ary reflected Gray sequence.
In view of the induction stage, we assume that the finite sequence $ h^{n-1,k-1}$ is a  $\sigma_{k-1}$-Gray cycle over $A^{n-1}$, for some $k\ge 2$.\\
(i) We start by proving that $h^{n,k}$ satisfies Condition  \ref{ivb}. This will be  done through the three following steps:

(i.i) Firstly, we prove that, for each $q\in [0,p-1]$, in  the subsequence $C_q$ 
two consecutive terms  are necessarily in correspondence under $\sigma_k$.
Given $r\in [0, p^{n-1}-1]$, by definition, we have $\theta^{r+q}(0)\in\sigma_1\left(\theta^{r+q-1}(0)\right)$.
Since $ h^{n-1,k-1}$ satisfies Condition \ref{ivb},  we have  $ h^{n-1,k-1}_{[r]}\in\sigma_{k-1}\left( h^{n-1,k-1}_{[r-1]}\right)$.
We obtain $\theta^{r+q}(0) h^{n-1,k-1}_{[r]}\in\sigma_k\left(\theta^{r+q-1}(0) h^{n-1,k-1}_{[r-1]}\right)$, thus  according to (\ref{Gamma-A-ge3}):
$h^{n,k}_{[qp^{n-1}+r]}\in\sigma_k\left( h^{n,k}_{[qp^{n-1}+r-1]}\right)$.

(i.ii) Secondly, we prove that, for each $q\in [1,p-1]$, the last term of $C_{q-1}$ 
and the initial term of $C_{q}$ are also connected by $\sigma_k$. 
Take $r=0$ in Eq.  (\ref{Gamma-A-ge3}): 
it follows from $\theta^{p^{n-1}}=id_A$ that we have
$h^{n,k}_{[qp^{n-1}]}=\theta^{q}(0) h^{n-1,k-1}_{[0]}=\theta^{p^{n-1}+q}(0) h^{n-1,k-1}_{[0]}$.
In (\ref{Gamma-A-ge3}) take $r=p^{n-1}-1$, moreover substitute $q-1\in [0,p-2]$ to $q\in [1,p-1]$: we obtain
$h^{n,k}_{[qp^{n-1}-1]}=\theta^{q+p^{n-1}-2}(0) h^{n-1,k-1}_{[p^{n-1}-1]}$.
It follows from $p=|A|\ge 3$ that $\theta(0)\neq\theta^{-2}(0)$: since $\theta$ is one-to-one this implies
$\theta^{q+p^{n-1}}(0)\neq \theta^{q+p^{n-1}-2}(0)$, thus  $\theta^{q+p^{n-1}}(0)\in\sigma_1\left(\theta^{q+p^{n-1}-2}(0)\right)$.
By induction 
we have $ h^{n-1,k-1}_{[0]}=\sigma_{k-1}\left( h^{n-1,k-1}_{[p^{n-1}-1]}\right)$, thus
$h^{n,k}_{[qp^{n-1}]}\in\sigma_k\left(\theta^{q+p^{n-1}-2}(0) h^{n-1,k-1}_{[qp^{n-1}-1]}\right)$ that is,
 $h^{n,k}_{[qp^{n-1}]}\in\sigma_k\left( h^{n,k}_{[qp^{n-1}-1]}\right)$.

(i.iii)
At last, we  prove that the first term of $C_0$ is an image under $\sigma_k$ of the last term of $C_{p-1}$.
In  Eq. (\ref{Gamma-A-ge3}), take  $q=0$ and $r=0$: we obtain $h^{n,k}_{[0]}=0h^{n-1,k-1}_{[0]}$.
Similarly, by setting $q=p-1$ and $r=p^{n-1}-1$,
we obtain $h^{n,k}_{[(p-1)p^{n-1}+p^{n-1}-1]}=\theta^{p^{n-1}+p-2}(0) h^{n-1,k-1}_{[p^{n-1}-1]}$,
thus $h^{n,k}_{[p^n-1]}=\theta^{-2}(0)h^{n-1,k-1}_{[p^{n-1}-1]}$.
Since $ h^{n-1,k-1}$  is a $\sigma_{k-1}$-Gray cycle over $A^{n-1}$, we have $ h^{n-1,k-1}_{[0]}\in\sigma_{k-1}\left( h^{n-1,k-1}_{[p^{n-1}-1]}\right)$.
In addition, it follows from $p\ge 3$, that $\theta^{-2}(0)\neq 0$, thus $0\in\sigma_1\left(\theta^{-2}(0)\right)$. 
We obtain
$ h^{n,k}_{[0]}
\in\sigma_k\left(\theta^{-2}(0)h^{n-1,k-1}_{[p^{n-1}-1]}\right)$, thus  $h^{n,k}_{[0]}\in\sigma_k\left( h^{n,k}_{[p^{n}-1]}\right)$ that is, the required property.\\
(ii) Now, we prove that,  
in the sequence $ h^{n,k}$ all terms are pairwise different.
Let $i,i'\in [0,p^n-1]$ such that $h^{n,k}_{[i]}= h^{n,k}_{[i']}$ and consider the unique  $4$-tuple  of integers
$q,q'\in[0,p-1]$, $r,r'\in [0,p^{n-1}-1]$ such that $i=qp^{n-1}+r$ and $i'=q' p^{n-1}+r'$. 
According to  (\ref{Gamma-A-ge3}) we have $\theta^{q+r}(0)h^{n-1,k-1}_{[r]}=\theta^{q'+r'}(0)h^{n-1,k-1}_{[r']}$, thus
$\theta^{q+r}(0)=\theta^{q'+r'}(0)\in A$ and $h^{n-1,k-1}_{[r]}=h^{n-1,k-1}_{[r']}$.
Since $h^{n-1,k-1}$ satisfies \ref{ivc}, the second equation implies $r=r'$, whence
 the first one implies $\theta^{q}(0)=\theta^{q'}(0)$, thus  $q=q'\bmod p$. Since we have $q,q'\in [0,p-1]$ we obtain $q=q'$, thus  $i=i'$.\\
(iii)  Finally, since $ h^{n,k}$ satisfies  \ref{ivc}, we have $\left|\bigcup_{0\le i\le p^{n}-1}\{h^{n,k}_{[i]}\}\right|=p^{n}$, hence $ h^{n,k}$ satisfies Condition \ref{iva}. 
~~~~~~~~~~~~~~~~~~~~~~~~~~~~~~~~~~~~~~~~~~~~~~~~~~~~~~~~~~~~~~~~~~~~~~~~$\square$
\end{proof}
\section{The case where $A$ is a binary alphabet, with $k$ odd}
\label{S4}
Let $A=\{0,1\}$ and $n\ge k$.
Classically, the cyclic permutation  $\theta$, which was introduced in  Sect. \ref{Prelim}, can be extended into a one-to-one monoid homomorphism onto $A^*$:
in view of this, we set $\theta(\varepsilon)=\varepsilon$ and,  for any non-empty $n$-tuple of characters $a_1,\cdots, a_n\in A$, $\theta(a_1\cdots a_n)=\theta(a_1)\cdots\theta(a_n)$.
Trivially, in the case where we have $n=k$, if a non-empty $\sigma_k$-Gray code exists over $X\subseteq A^n$, then  we have $X=\{x,\theta(x)\}$, for some $x\in A^n$.
In the sequel of the paper, we assume $n\ge k+1$. 
In what follows, we indicate the construction of a peculiar  pair of  $\sigma_k$-Gray cycles over $A^n$, namely $\gamma^{n,k}$ and $\rho^{n,k}$.
This will be done by induction over $k'$, the unique non-negative integer such that  $k=2k'+1$. Let $n_0=n-2k'=n-k+1$.\\
-- For the base case, $\gamma^{n_0,1}$ and $\rho^{n_0,1}$ are computed by applying  some
reversal (resp., shift) over the sequence $g^{n_0,1}$ from   Sect. \ref{Prelim}: 
\begin{eqnarray}
\label{gamma-init}
\gamma^{n_0,1}_{[0]}=g^{n_0,1}_{[0]}~~{\rm and}~~
\gamma^{n_0,1}_{[i]}=g^{n_0,1}_{[2^{n_0}-i]}~~(1\le i\le 2^{n_0}-1);\\
\label{gamma-init1}
\rho^{n_0,k}_{[0]}=g^{n_0,1}_{[2^{n_0}-1]}~~{\rm and}~~
\rho^{n_0,k}_{[i]}=g^{n_0,1}_{[i-1]}~~(1\le i\le 2^{n_0}-1).
\end{eqnarray}
By construction,  $\gamma^{n_0,1}$ and $\rho^{n_0,1}$ are $\sigma_1$-Gray cycles over $A^{n_0}$. Moreover we have:
\begin{eqnarray}
\label{termes-initiaux}
\gamma^{n_0,1}_{[0]}=g^{n_0,1}_{[0]}=0^{n_0}~~{\rm and}~~\rho^{n_0,1}_{[0]}=g^{n_0,1}_{[2^{n_0}-1]}=10^{n_0-1};\\
\label{termes-initiaux1}
\gamma^{n_0,1}_{[2^{n_0}-1]}=g^{n_0,1}_{[1]}=0^{n_0-1}1~~{\rm and}~~\rho^{n_0,1}_{[2^{n_0}-1]}=g^{n_0,1}_{[2^{n_0}-2]}=10^{n-2}1.
\end{eqnarray}
\begin{example} 
\label{Ex3} 
For $n_0=3$ we obtain the following sequences:
$$
{\small
\begin{array} {c}
g^{3,1}\\ \overbrace{~~}\\ 000\\ 001\\ 011\\ 010\\ 110\\ 111\\ 101\\ 100
\end{array}
~~~~~~~~~~~~~~~~~
\begin{array} {c}
\gamma^{3,1}\\ \overbrace{~~}\\ 000\\ 100\\ 101\\ 111\\ 110\\ 010\\ 011\\ 001
\end{array}
~~~~~~~~~~~~~~~~~
\begin{array} {c}
\rho^{3,1}\\ \overbrace{~~}\\ 100\\ 000\\ 001\\ 011\\ 010\\ 110\\ 111\\ 101
\end{array}
}
$$
\end{example}
-- In view of the induction step, we assume that we have computed  the $\sigma_k$-Gray cycles $\gamma^{n,k}$ and $\rho^{n,k}$. 
Notice that we have $n+2=n_0+2(k'+1)=n_0+(k+2)-1$:
below we explain  the construction of  the two corresponding $2^{n+2}$-term sequences $\gamma^{n+2,k+2}$ and $\rho^{n+2,k+2}$. 
Let $i\in [0,2^{n+2}-1]$, and let $q\in [0,3]$, $r\in [0,2^{n}-1]$ be the unique pair of integers such that $i=q2^n+r$.
Since we have $r\in [0,2^{n}-1]$, taking for $q$ the value  $q=0$ (resp., $1$, $2$, $3$), we state the corresponding equation (\ref{6}) {\large (}resp., (\ref{7}),(\ref{8}),(\ref{9}) {\large)}:
\begin{subeqnarray} 
\gamma^{n+2,k+2}_{[r]}=\theta^{r}(00)\gamma^{n,k}_{[r]}; \slabel{6}\\
\gamma^{n+2,k+2}_{[2^n+r]}=\theta^{r}(01)\rho^{n,k}_{[r]}; \slabel{7}\\
\gamma^{n+2,k+2}_{[2.2^n+r]}=\theta^{r}(11)\gamma^{n,k}_{[r]}; \slabel{8}\\
\gamma^{n+2,k+2}_{[3.2^n+r]}=\theta^{r}(10)\rho^{n,k}_{[r]}. \slabel{9}
\end{subeqnarray}
Similarly the sequence $\rho^{n+2,k+2}$ is computed by substituting,  in the preceding equations, the $4$-tuple 
$(10,11,01,00)$ to $(00,01,11,10)$:
\begin{subeqnarray} 
\rho^{n+2,k+2}_{[r]}=\theta^{r}(10)\gamma^{n,k}_{[r]}; \slabel{10}\\
\label{H2} \rho^{n+2,k+2}_{[2^n+r]}=\theta^{r}(11)\rho^{n,k}_{[r]}; \slabel{11}\\
\label{H3} \rho^{n+2,k+2}_{[2.2^n+r]}=\theta^{r}(01)\gamma^{n,k}_{[r]}; \slabel{12}\\
\label{H4} \rho^{n+2,k+2}_{[3.2^n+r]}=\theta^{r}(00)\rho^{n,k}_{[r]}.\slabel{13}
\end{subeqnarray}
\begin{example}
(Example \ref{Ex3} continued)
$\gamma^{5,3}$  is the concatenation,  in this order, of  the $4$  following  subsequences:
$$
{\small
\begin{array} {c}
~~~~~{\scriptstyle\gamma^{3,1}}\\ ~~~~\overbrace{}\\
 00~000\\ 11~100\\ 00~101\\ 11~111\\ 00~110\\ 11~010\\00~011\\ 11~001
\end{array}
~~~~~~~~~~~~~~~~~
\begin{array} {c}
~~~~~~~{\scriptstyle\rho^{3,1}}\\ ~~~~\overbrace{~~~~~ }\\
01~100\\ 10~000\\ 01~001\\ 10~011\\01~010\\ 10~110\\01~111\\ 10~101
\end{array}
~~~~~~~~~~~~~~~~~
\begin{array} {c}
~~~~~~~{\scriptstyle\gamma^{3,1}}\\ ~~~~\overbrace{~~~~~ }\\
11~000\\ 00~100\\ 11~101\\ 00~111\\ 11~110\\ 00~010\\11~011\\ 00~001
\end{array}
~~~~~~~~~~~~~~~~~
\begin{array} {c}
~~~~~~~{\scriptstyle\rho^{3,1}}\\ ~~~~\overbrace{~~~~~ }\\10~100\\ 01~000\\ 10~001\\ 01~011\\10~010\\ 01~110\\10~111\\ 01~101
\end{array}
}
$$
\end{example}
\begin{lemma}
\label{pairwise-different}
$\gamma^{n,k}$ and 
$\rho^{n,k}$ satisfy both the conditions \ref{iva} and \ref{ivc}.
\end{lemma}
\begin{proof}
We argue by induction over  $k'\ge 0$, with $k=2k'+1$. 
the base case corresponds to  $k'=0$ that is,  $k=1$ and $n=n_0$: 
as indicated above, $\gamma^{n_0,1}$ and $\rho^{n_0,1}$ are $\sigma_1$-Gray cycles over $A^{n}$.
In view of  the induction step we assume that, for some $k'\ge 0$, both the sequences 
$\gamma^{n,k}$ and $\rho^{n,k}$
are $\sigma_k$-Gray cycles over $A^n$.\\
(i) In order to prove that $\gamma^{n+2,k+2}$ satisfies Condition \ref{ivc},
let $i,i'\in [0,2^{n+2}-1]$ such that 
 $\gamma^{n+2,k+2}_{[i]}=\gamma^{n+2,k+2}_{[i']}$, and
$q,q'\in [0,3]$, $r,r'\in [0,2^n-1]$ such that $i=q2^n+r$, $i'=q'2^n+r'$.
According to Eqs. (\ref{6})--(\ref{9}), words $x,x'\in A^2$, $w,w'\in A^n$ exist such that
$\gamma^{n+2,k+2}_{[i]}=\theta^r(x)w$ and $\gamma^{n+2,k+2}_{[i']}=\theta^{r'}(x')w'$ that is,
$\theta^r(x)=\theta^{r'}(x')\in A^2$ and $w=w'$. By the definition of $\theta$, this implies either $x,x'\in\{00,11\}$ or $x,x'\in\{01,10\}$ that is, by construction,
either $q,q'\in\{0,2\}$, $x,x'\in\{00,11\}$, $w=\gamma^{n,k}_{[r]}=\gamma^{n,k}_{[r']}$, or 
$q,q'\in\{1,3\}$, $x,x'\in\{01,10\}$, $w=\rho^{n,k}_{[r]}=\rho^{n,k}_{[r']}$.
Since $\gamma^{n,k}$ and $\rho^{n,k}$ satisfies  \ref{ivc}, in any case we have  $r=r'$.
This implies  $\theta^r(x)=\theta^{r}(x')$, thus $x=x'$. 
With regard  to  Eqs. (\ref{6})--(\ref{9}), this corresponds to  $q=q'$, thus  $i=q2^n+r=q'2^n+r=i'$, therefore $\gamma^{n+2,k+2}$ satisfies Condition \ref{ivc}.\\
(ii) By substituting  $(10,11,01,00)$ to  $(00,01,11,10)$, according to (\ref{10})--(\ref{13}),
similar arguments prove that $\rho^{n+2,k+2}_{[i]}=\rho^{n+2,k+2}_{[i']}$ implies $i=i'$, thus  $\rho^{n+2,k+2}$ also satisfies \ref{ivc}.\\
 (iii) Since $\gamma^{n+2,k+2}$ satisfies \ref{ivc}, we have
$\bigcup_{0\le i\le 2^{n+2}-1}\{\gamma^{n+2,k+2}_i\}=A^{n+2}$, hence 
it satisfies \ref{iva}.
Similarly, since  $\rho^{n+2,k+2}$ satisfies \ref{ivc} 
it  satisfies \ref{iva}.
$\square$
\end{proof}
{\flushleft In order to prove that our sequences satisfy \ref{ivb}, we prove the following property:}
\begin{lemma}
\label{For-Gray-connections}
We have $\gamma^{n,k}_{[0]}\in\sigma_{k+1}\left(\rho^{n,k}_{[2^{n}-1]}\right)$ and 
$\rho^{n,k}_{[0]}\in\sigma_{k+1}\left(\gamma^{n,k}_{[2^{n}-1]}\right).$
\end{lemma}
\begin{proof}
We argue by induction over the integer $k'\ge 0$. 
The case $k'=0$ corresponds to $k=1$ and $n=n_0$: with such a condition, our property  comes from  the identities (\ref{termes-initiaux}) and  (\ref{termes-initiaux1}).
For the induction step, we assume that, for some $k'\ge 0$, we have $\gamma^{n,k}_{[0]}\in\sigma_{k+1}\left(\rho^{n,k}_{[2^{n}-1]}\right)$ and $\rho^{n,k}_{[0]}\in\sigma_{k+1}\left(\gamma^{n,k}_{[2^{n}-1]}\right)$.\\
(i) In  (\ref{6}), by taking $r=0$ we obtain
$\gamma^{n+2,k+2}_{[0]}=00\gamma^{n,k}_{[0]}$, hence by induction:
$\gamma^{n+2,k+2}_{[0]} \in 00\sigma_{k+1}\left(\rho^{n,k}_{[2^{n}-1]}\right)\subseteq \sigma_{k+3}\left(11\rho^{n,k}_{[2^{n}-1]}\right)$.
By setting $r=2^{n}-1$ in  (\ref{13}), we obtain $\rho^{n+2,k+2}_{[2^{n+2}-1]}=11\rho^{n,k}_{[2^{n}-1]}$, thus
$\gamma^{n+2,k+2}_{[0]}\in\sigma_{k+3}\left(\rho^{n+2,k+2}_{[2^{n+2}-1]}\right)$.\\
(ii) Similarly, by setting $r=0$  in  (\ref{10}), and  by induction we have:
$\rho^{n+2,k+2}_{[0]}=10\gamma^{n,k}_{[0]}\in\sigma_{k+3}\left(01\rho^{n,k}_{[2^{n}-1]}\right)$.
By taking $r=2^n-1$ in  (\ref{9}) we obtain $\gamma^{n+2,k+2}_{[2^{n+2}-1]}=01\rho^{n,k}_{[2^n-1]}$, thus  $\rho^{n+2,k+2}_{[0]}\in\sigma_{k+3}\left(\gamma^{n+2,k+2}_{[2^{n+2}-1]} \right)$.
~~~~~~~~~~~~~~~~~~~~~~~~~~~~~~~~~~~~~~~~~~~~~~~~~~~~~~~~~~~~~~~~~~~~$\square$
\end{proof}
Since Eqs.  (\ref{6})--(\ref{13}) look alike, one may be tempted to compress them thanks to some unique generic formula.
Based on our tests, such a formula needs to introduce at least two additionnal technical parameters, which would make their handling tedious. 
In the proof of the following result, we have opted to report some case-by-case basis argumentation:
this has the advantage of making use of arguments which, although being similar, actually are easily readable.
\begin{proposition}
\label{Gamma-n-k-odd}
 Both  the sequences $\gamma^{n,k}$ and $\rho^{n,k}$
are $\sigma_k$-Gray cycles over $A^n$.
\end{proposition}
{\it Proof sketch}
 For the full proof, once more we argue by induction over $k'\ge 0$.
Since $\gamma^{n_0,1}$ and $\rho^{n_0,1}$ are  $\sigma_1$-Gray cycles over $A^n$, the property holds for $k'=0$.
In view of  the induction stage, we assume that, for some $k'\ge 0$ both the sequences $\gamma^{n,k}$ and $\rho^{n,k}$ are  $\sigma_k$-Gray cycles over $A^n$.
According to  Lemma \ref{pairwise-different}, it remains to establish that $\gamma^{n+2,k+2}$ and $\rho^{n+2,k+2}$ satisfy Condition \ref{ivb} that is:
\begin{eqnarray}
\slabel{EQ1}
(\forall q\in \{0,1,2,3\})(\forall r\in [1,2^n-1]) ~\gamma^{n+2,k+2}_{[q2^n+r]}\in\sigma_{k+2}\left(\gamma^{n+2,k+2}_{[q2^n+r-1]}\right);\\
\slabel{EQ2}
(\forall q\in \{1,2,3\})~\gamma^{n+2,k+2}_{[q2^{n}]}\in 
\sigma_{k+2}\left(\gamma^{n+2,k+2}_{[q2^{n}-1]}\right);\\
\slabel{EQ3}\gamma^{n+2,k+2}_{[0]}\in\sigma_{k+2}\left(\gamma^{n+2,k+2}_{[2^{n+2}-1]}\right).
\end{eqnarray}
\begin{eqnarray}
\label{EQ4}
(\forall q\in \{0,1,2,3\})(\forall r\in [1,2^n-1]) ~\rho^{n+2,k+2}_{[q2^n+r]}\in\sigma_{k+2}\left(\rho^{n+2,k+2}_{[q2^n+r-1]}\right);\\
\label{EQ5}
(\forall q\in \{1,2,3\})~\rho^{n+2,k+2}_{[q2^{n}]}\in 
\sigma_{k+2}\left(\rho^{n+2,k+2}_{[q2^{n}-1]}\right);\\
\label{EQ6}
\rho^{n+2,k+2}_{[0]}\in\sigma_{k+2}\left(\rho^{n+2,k+2}_{[2^{n+2}-1]}\right).
\end{eqnarray}
\begin{itemize}[label=$$, leftmargin =-0.1cm]
\item \underline{{\it Condition} (\ref{EQ1}).}
(i) At first assume $q=0$.
According to (\ref{6}), and since $\gamma^{n,k}$ satisfies \ref{ivb},
we have  $\gamma^{n+2,k+2}_{[r]}= \theta^{r}(00)\gamma^{n,k}_{[r]}\in\theta^{r}(00)\sigma_k\left(\gamma^{n,k}_{[r-1]}\right)$,
thus $\gamma^{n+2,k+2}_{[r]}\in\sigma_{k+2}\left(\theta^{r-1}(00)\gamma^{n,k}_{[r-1]}\right)$. 
In  (\ref{6}), substitute $r-1$ to $r$ (we have $0\le r-1\le 2^n-2$):
we obtain  $\gamma^{n+2,k+2}_{[r-1]}=\theta^{r-1}(00)\gamma^{n,k}_{[r-1]}$, thus
$\gamma^{n+2,k+2}_{[r]}\in\sigma_{k+2}\left(\gamma^{n+2,k+2}_{[r-1]}\right)$.\\
(ii)  Now assume   $q=1$.
According to (\ref{7}), and since $\rho^{n,k}$ satisfies \ref{ivb},
we have  $\gamma^{n+2,k+2}_{[2^n+r]}= \theta^{r}(01)\rho^{n,k}_{[r]}\in \sigma_{k+2}\left(\theta^{r-1}(01)\rho^{n,k}_{[r-1]}\right)$.
In (\ref{7}), susbtitute $r-1$ to $r$: we obtain $\gamma^{n+2,k+2}_{[2^n+r-1]}= \theta^{r-1}(01)\rho^{n,k}_{[r-1]}$, thus
$\gamma^{n+2,k+2}_{[2^n+r]}\in\sigma_{k+2}\left(\gamma^{n+2,k+2}_{[2^n+r-1]}\right)$.\\
(iii) For $q=2$, the arguments are similar to those applied in (i) by substituting $\gamma^{n+2,k+2}_{[2\cdot 2^n+r]}$ to $\gamma^{n+2,k+2}_{[r]}$,
 Eqs. (\ref{8}) to  (\ref{6}),
 and $11$ to $00$.\\
(iv)  Similarly, for  $q=3$ the proof is obtained by substituting in (ii)   $\gamma^{n+2,k+2}_{[3\cdot2^{n}+r]}$ to $\gamma^{n+2,k+2}_{[2^n+r]}$, 
(\ref{9}) to  (\ref{7}),
and $10$ to $01$.%
\item
\underline{{\it Condition} (\ref{EQ2}).}
(i) Assume $q=1$ and take $r=0$ in  (\ref{7}). According to Lemma \ref{For-Gray-connections},
we obtain $\gamma^{n+2,k+2}_{[2^n]}=01\rho^{n,k}_{[0]}\in 01\sigma_{k+1}\left(\gamma^{n,k}_{[2^n-1]}\right)\subseteq\sigma_{k+2}\left(11\gamma^{n,k}_{[2^n-1]}\right)$.
Take $r=2^n-1$ in  (\ref{6}):
we obtain $\gamma^{n+2,k+2}_{[2^n-1]}=11\gamma^{n,k}_{[2^n-1]}$, thus
$\gamma^{n+2,k+2}_{[2^n]}\in\sigma_{k+2}\left(\gamma^{n+2,k+2}_{[2^n-1]}\right)$.\\
(ii) Now, assume  $q=2$, and 
set $r=0$ in Eq. (\ref{8}). According to Lemma \ref{For-Gray-connections} 
we have $\gamma^{n+2,k+2}_{[2\cdot2^n]}=11\gamma^{n,k}_{[0]}\in11\sigma_{k+1}\left(\rho^{n,k}_{[2^n-1]}\right)\subseteq \sigma_{k+2}\left(10\rho^{n,k}_{[2^n-1]}\right)$.
By taking $r=2^n-1$ in (\ref{7}), we obtain $\gamma^{n+2,k+2}_{[2\cdot 2^n-1)]}=10\rho^{n,k}_{[2^n-1]}$, thus
$\gamma^{n+2,k+2}_{[2\cdot2^n]}\in\sigma_{k+2}\left(\gamma^{n+2,k+2}_{[2\cdot2^n-1]}\right)$.\\
%
%
(iii)  For  $q=3$, substitute  in (i) $\gamma^{n+2,k+2}_{[3\cdot 2^n]}$ to $\gamma^{n+2,k+2}_{[2^n]}$,
Eq. (\ref{9}) to  Eq. (\ref{7}),  (\ref{8}) to  (\ref{6}), $10$ to $01$ and $00$ to $11$: similar arguments  prove that
$\gamma^{n+2,k+2}_{[3\cdot2^n]}\in\sigma_{k+2}\left(\gamma^{n+2,k+2}_{[3\cdot2^n-1]}\right)$. \\
\item
\underline{{\it Condition} (\ref{EQ3}).}
Take $r=0$  in  (\ref{6}). According to Lemma \ref{For-Gray-connections}, we have  $\gamma^{n+2,k+2}_{[0]}=00\gamma^{n,k}_{[0]}\in 00\sigma_{k+1}\left(\rho^{n,k}_{[2^n-1]}\right)\subseteq\sigma_{k+2}\left(01\rho^{n,k}_{[2^n-1]}\right)$. 
By taking $r=2^n-1$ in (\ref{9}) we obtain 
$\gamma^{n,k}_{[2^{n+2}-1]}=01\rho^{n,k}_{[2^n-1]}$, thus
$\gamma^{n+2,k+2}_{[0]}\in\sigma_{k+2}\left(\gamma^{n,k}_{[2^{n+2}-1]}\right)$.
\end{itemize}
{\flushleft According to the structures of Eqs. (\ref{10})--(\ref{13}), for proving the conditions (\ref{EQ4})--(\ref{EQ6}),}
the method  consists in substituting
the word $\rho^{n+2,k+2}_{[r]}$ to $\gamma^{n+2,k+2}_{[r]}$, the $4$-uple $(10,11,01,00)$ to $(00,01,11,10)$, and 
Eq.   (\ref{10}) {\large (}resp.,  (\ref{11}),  (\ref{12}),  (\ref{13}){\large )} to Eq. (\ref{6})  {\large (}resp.,  (\ref{7}),  (\ref{8}),  (\ref{9}){\large )}.
~~~~~~~~~~~~~~~~~~~~~~~~~~~~~~~~~~~~~~~~~~~~~~~~$\Box$
\section{The case where we have $|A|=2$ and  $k$ even}
\label{Cons}
Beforehand,  we remind some classical algebraic interpretation of the substitution $\sigma_k$  in the framework of the binary alphabet $A=\{0,1\}$.
Denote by $\oplus$ the addition in the group ${\mathbb Z}/2{\mathbb Z}$ with identity $0$. 
Given a positive integer $n$, 
and  $w,w'\in A^n$, define $w\oplus w'$ as the unique word of $A^n$ such that, for each $i\in [1,n]$:
$(w\oplus w')_i=w_i\oplus w'_i$. With this notation the sets $A^n$ and $({\mathbb Z}/2{\mathbb Z})^n$ are in one-to-one correspondence. Moreover we have
$w'\in\sigma_k(w)$ if, and only if, some word $u\in A^n$ exists such that $|u|_1=k$ and $w=w' {\oplus} u$, therefore if $k$ is even  we have $|w|_1=|w'|_1\bmod 2$.
Consequently, given a  $\sigma_k$-Gray cycle $\left(\alpha_{[i]}\right)_{0\le i\le m}$, 
for each $i\in [0,m]$ we have $\left|\alpha_{[i]}\right|_1=\left|\alpha_{[0]}\right|_1\bmod 2$.
As a corollary, setting ${\rm Even}_1^n=\{w\in A^*:|w|_1=0 \bmod 2\}$ and ${\rm Odd}_1^n=\{w\in A^*:|w|_1=1\bmod 2\}$:
\begin{lemma}
\label{parity}
With the condition of  Sect. \ref{Cons},  given a  $\sigma_k$-Gray cycle $\alpha$ over $X$, either we have $X\subseteq {\rm Even}_1^n$, or we have $X\subseteq {\rm Odd}_1^n$.
\end{lemma}
Since $k-1$ is an odd integer, according to Proposition \ref{Gamma-n-k-odd}, the sequence $\gamma^{n-1,k-1}$
is a  $\sigma_{k-1}$-Gray cycle over $A^{n-1}$. We set:
\begin{eqnarray}
\label{Construct}
(\forall i\in [0, 2^{n-1}-1])~~\gamma^{n,k}_{[i]}=\theta^i(0)\gamma^{n-1,k-1}_{[i]}~~{\rm and}
~~\underline{\gamma}^{n,k}_{[i]}=\theta^i(1)\gamma^{n-1,k-1}_{[i]}
\end{eqnarray}
For instance, we have  $\gamma^{6,4}_{[0]}=0 00 000$, $\underline\gamma^{6,4}_{[0]}=1 00 000$, $\gamma^{6,4}_{[1]}=111100$..

\begin{proposition}
\label{Gamma-n-k-even}
 $\gamma^{n,k}$  (resp., $\underline{\gamma}^{n,k}$) is a $\sigma_k$-Gray cycle over  ${\rm Even}_1^{n}$ (resp., ${\rm Odd}^{n}_1$).
\end{proposition}
\begin{proof}
(i) According to Eq. (\ref{Construct}), 
 since $\gamma^{n-1,k-1}$ satisfies \ref{ivc}, both the sequences $\gamma^{n,k}$ and $\underline{\gamma}^{n,k}$ also satisfy \ref{ivc}.\\
(ii) By Lemma \ref{parity}, we have $\bigcup_{0\le i\le 2^{n}-1}\left\{\gamma^{n,k}\right\}\subseteq{\rm Even}_1^{n}$
and $\bigcup_{0\le i\le 2^{n}-1}\left\{\underline{\gamma}^{n,k}\right\}\subseteq {\rm Odd}_1^{n}$.
In addition, according to (\ref{Construct}), we have  $\left|\gamma^{n,k}\right|=\left|\underline{\gamma}^{n,k}\right|=\left|\gamma^{n-1,k-1}\right|=2^{n-1}$. This implies 
 $\bigcup_{0\le i\le 2^{n}-1}\left\{\gamma^{n,k}\right\}= {\rm Even}_1^{n}$
and $\bigcup_{0\le i\le 2^{n}-1}\left\{\underline{\gamma}^{n,k}\right\}= {\rm Odd}_1^{n}$ that is, $\gamma^{n,k}$ and $\underline\gamma^{n,k}$ satisfy \ref{iva}.\\
(iii) Let $i\in [1, 2^{n-1}-1]$. Since $\gamma^{n-1,k-1}$ satisfies \ref{ivb}, we have $\gamma^{n-1,k-1}_{[i]}\in\sigma_{k-1}\left(\gamma^{n-1,k-1}_{[i-1]}\right)$.
According to (\ref{Construct}), the initial characters of $\gamma^{n,k}_{[i]}$
and  $\gamma^{n,k}_{[i-1]}$ (resp.,  $\underline\gamma^{n,k}_{[i]}$
and  $\underline\gamma^{n,k}_{[i-1]}$) are different, hence we have $\gamma^{n,k}_{[i]}\in\sigma_k\left(\gamma^{n,k}_{[i-1]}\right)$ and $\underline\gamma^{n,k}_{[i]}\in\sigma_k\left(\underline\gamma^{n,k}_{[i-1]}\right)$.
In addition, once more according to (\ref{Construct}) it follows from $\gamma^{n-1,k-1}_{[0]}\in\sigma_{k-1}\left(\gamma^{n-1,k-1}_{[2^{n-1}-1]}\right)$
that $\gamma^{n,k}_{[0]}=0\gamma^{n-1,k-1}_{[0]}\in\sigma_{k}\left(1\gamma^{n-1,k-1}_{[2^{n-1}-1]}\right)\subseteq\sigma_k\left(\gamma^{n,k}_{[2^{n-1}-1]}\right)$, hence $\gamma^{n,k}$ satisfies \ref{ivb}.
Similarly,   $\underline\gamma^{n-1,k-1}_{[0]}\in\sigma_{k-1}\left(\underline\gamma^{n-1,k-1}_{[2^{n-1}-1]}\right)$ implies
$\underline\gamma^{n,k}_{[0]}\in\sigma_k\left(\gamma^{n,k}_{[2^{n-1}-1]}\right)$, hence $\underline\gamma^{n,k}$ satisfies \ref{ivb}.
~~~~~~~~~~~~~~~~~~~~~~~~~~~~~~~~~~~~~~~~~~~~~~~~~~~~~~~~~~~~~~$\square$
\end{proof}
{\flushleft The following statement provides the description of the  complexity $\lambda_{A,\tau}$:}
\begin{theorem} 
\label{H-maximal}
Given a finite alphabet $A$ and $n\ge k\ge 1$, exactly one of the four following properties  holds:

{\rm (i)} $|A|\ge 3$, $n\ge k$, and $\lambda_{A,\sigma_k}(n)=|A|^n$;

{\rm (ii)}  $|A|=2$,   $n=k$, and $\lambda_{A,\sigma_k}(n)=2$;

{\rm (iii)}  $|A|=2$,  $n\ge k+1$, $k$ is odd and  $\lambda_{A,\sigma_k}(n)=2^n$;

{\rm (iv)} $|A|=2$, $n\ge k+1$, $k$ is even, and $\lambda_{A,\sigma_k}(n)=2^{n-1}$.\\
In addition, in each case some $\sigma_k$-Gray cycle of maximum length can be explicitly computed.
\end{theorem}
\begin{proof}
Recall that if some  $\sigma_k$-Gray cycle exists over $X\subseteq A^{\le n}$, necessarily $X$ is a uniform set that is, $X\subseteq A^m$ holds for some $m\ge n$;
hence,  in any case we have $\lambda_{A,\sigma_k}(n)\le |A|^n$.
According to Proposition \ref{Cycle-Age3-n-k}, if we have $|A|\ge 3$ and  $n\ge k$, a $\sigma_k$-Gray cycle exists over $A^n$, hence Property (i) holds.
Similarly, (iii) comes from Proposition \ref{Gamma-n-k-odd}. 
As indicated in the preamble of  Sect. \ref{A>3}, Property {\rm (ii)} trivially holds. Finally, according to  Lemma \ref{parity},
given a binary alphabet  $A$, if $k$ is  even  we have $\lambda_{A,\sigma_k}(n)\le 2^{n-1}$, hence (iv) comes from 
 Proposition \ref{Gamma-n-k-even}.~~~~~~~~~~~~~~~~~~~~~~~~~~~$\square$
\end{proof}
{\bf Further development}
Since our  Gray cycles  were constructed by applying recursive processes,  it is legitimate to ask whether some method could exist for computing $\gamma^{n,k}_{[i]}$ by directly starting with $\gamma^{n,k}_{[i-1]}$, 
as in the case of the classical reflected Gray cycles. 
In view of some of our more recent studies, we strongly believe that such algorithms can actually be devised: we hope to develop this point  in a further paper.
 
On the other hand, it could be of interest to study the behaviour of  $\Lambda_{A,\tau}$ in the framework of other word  binary relations $\tau$, 
even in restraining to special families of sets $X\subseteq A^*$, such as variable-length codes.
\bibliographystyle{splncs04}
\bibliography{mysmallbib}
\end{document}